\documentclass[default]{sn-jnl}

\usepackage{graphicx}%
\usepackage{multirow}%
\usepackage{amsmath,amssymb,amsfonts}%
\usepackage{amsthm}%
\usepackage{mathrsfs}%
\usepackage[title]{appendix}%
\usepackage{xcolor}%
\usepackage{textcomp}%
\usepackage{manyfoot}%
\usepackage{booktabs}%
\usepackage{algorithm}%
\usepackage{algorithmicx}%
\usepackage{algpseudocode}%
\usepackage{listings}%

\theoremstyle{thmstyleone}%
\newtheorem{theorem}{Theorem}
\theoremstyle{thmstyletwo}%
\newtheorem{example}{Example}%

\theoremstyle{thmstylethree}%
\newtheorem{problem}{Problem}%
\newtheorem{lemma}{Lemma}%

\newtheorem{proof*}{Proof}%

\begin{document}

\title[Networks of Classical Conditioning Gates and Their Learning]{Networks of Classical Conditioning Gates and Their Learning}

\author*[1]{\fnm{Shun-ichi} \sur{Azuma}}\email{sazuma@i.kyoto-u.ac.jp}

\author[2]{\fnm{Dai} \sur{Takakura}}\email{ultradait@gmail.com}

\author[3]{\fnm{Ryo} \sur{Ariizumi}}\email{ryoariizumi@go.tuat.ac.jp}

\author[4]{\fnm{Toru} \sur{Asai}}\email{asai@nuem.nagoya-u.ac.jp}

\affil*[1]{\orgdiv{Graduate School of Informatics}, \orgname{Kyoto University}, \orgaddress{\street{Yoshida-honmachi, Sakyo-ku}, \city{Kyoto}, \postcode{606-8501},  \country{Japan}}}

\affil[3]{\orgdiv{Faculty of Engineering}, \orgname{Tokyo University of Agriculture and Technology}, \orgaddress{\street{2-24-16 Naka-cho}, \city{Koganei}, \postcode{184-8588},  \country{Japan}}}

\affil[4]{\orgdiv{Graduate School of Engineering}, \orgname{Nagoya University}, 
\orgaddress{\street{ Furo-cho, Chikusa-ku}, \city{Nagoya}, \postcode{464-8603},  \country{Japan}}}

\abstract{
Chemical AI is chemically synthesized artificial intelligence that has the ability of learning in addition to information processing. A research project on chemical AI, called the {\it Molecular Cybernetics Project}, was launched in Japan in 2021 with the goal of creating 
a molecular machine that can learn a type of conditioned reflex through 
the process called {\it classical conditioning}. 
If the project succeeds in developing such a molecular machine, the next step would be to configure a network of such machines to realize more complex functions. 
With this motivation, this paper develops a method for learning a desired function 
in the network of nodes each of which can implement 
classical conditioning. 
First, we present a model of classical conditioning, which is called 
here a {\it classical conditioning gate}. 
We then propose a learning algorithm for the network of classical conditioning gates.}




\maketitle

\section{Introduction}
The development of DNA nanotechnology has raised expectations 
for {\it chemical AI}. 
Chemical AI is chemically synthesized artificial intelligence that has the ability of learning in addition to information processing.
A research project on chemical AI, called the {\it Molecular Cybernetics Project}, 
was launched in Japan in 2021 \cite{sousetsu}, 
with the goal of establishing an academic field called {\it molecular cybernetics}.

One of the milestones of the project is to create a molecular machine that can 
learn a type of conditioned reflex through 
the process called {\it classical conditioning} \cite{Pavlov}. 
Classical conditioning is the process of acquiring a conditioned reflex by giving 
an conditioned stimulus with an unconditioned stimulus.
It was discovered by the well-known psychological experiment of ``Pavlov's dog'': 
\begin{itemize}
\item if one feeds a dog repeatedly while ringing of a bell,
then the dog will eventually begin to salivate at the sound of the bell;   
\item If one just rings a bell repeatedly without doing anything else for the dog that salivates at the sound of the bell, the dog will stop salivating at the sound of the bell,
\end{itemize}
which are respectively called the {\it acquisition} and {\it extinction}. 
This project attempts to create liposomes with different functions and combine them to artificially achieve a function similar to classical conditioning.

If the milestone is achieved, the next step would be to configure a network of such machines to realize more complex functions. Therefore, it is expected to establish a learning method for such a network. However, there exists no method because the learning has to be performed by the interaction of classical conditioning on the network, which is completely different from what is being considered these days.
In fact, neural networks are well known as a learning model, where    
learning is performed by adjusting weights between nodes \cite{Neural_book}, not by providing external inputs as in classical conditioning. On the other hand, Boolean networks \cite{Kauffman} 
are known as a model of the network of logic gates and 
their learning has been studied, e.g., in \cite{Sun2022}; 
but it also differs from learning by providing external inputs.

In this paper, we develop a method for learning a desired function 
in a network of nodes each of which can implement classical conditioning.
First, classical conditioning is modeled as a time-varying logic gate 
with two inputs and single output. The two inputs correspond to the feeding and bell ringing  in Pavlov's dog experiment, and the gate operates as either a YES gate or an OR gate 
at each time. 
The gate state, which is either YES or OR, is determined by how the inputs are given, 
in a similar manner to classical conditioning. 
The model is called here a {\it classical conditioning gate}.
Based on this model, the network of classical conditioning gates and 
its learning problem are formulated. 
We then derive a key principle to solving the problem,
called the {\it flipping principle}, that the gate state of any node in the network can 
be flipped while preserving the state of some other nodes.
By the flipping principle, we present a learning algorithm 
to obtain a desired function on the network.

Finally, we note the terminology and notation used in this paper.
We consider two types of logical gates, logical YES and logical OR.
Table~\ref{tab: truce table} shows the truth tables.   
We use  $x_1\vee x_2$ to represent the logical OR operation of the binary variables $x_1$ and $x_2$.
Moreover, let $\bigvee_{i\in {\mathbf I}} x_{i}$ denote the logical OR operation of $x_{i}$ 
with respect to  all the indeces in a finite set ${\mathbf I}$.

\begin{table}[tb]
 \caption{Truth tables of logical YES and OR.\vspace*{-2mm}}
 
 \label{tab: truce table}
  \begin{tabular}{|c|c||c|c|}
    \hline
     Input 1 &  Input 2  & Output (YES)  & Output (OR)   \\
    \hline  \hline
      0 &  0  &  0  &  0  \\
    \hline
      0 & 1   &  0  &  1  \\
    \hline
      1 & 0   &  1  &  1  \\
    \hline
      1 & 1   &  1  &  1  \\
    \hline
  \end{tabular}
\end{table}

\section{System Modeling}
\subsection{Classical Conditioning Gates}
We model classical conditioning
as shown in Fig.~\ref{fig: classical conditioning device}. 
It is a two-state machine that switches 
between the states ``YES'' and ``OR'' based on the two input signals taking a binary value.    
When the state is ``YES'', the gate operates as a logical YES gate,
whose output is equal to the first input, as shown in Table~\ref{tab: truce table}.  
On the other hand, 
when the state is ``OR'', 
the gate operates as a logical OR  gate,  
whose output is equal to 1 if and only if 
at least one of the two inputs is equal to 1. 
The state changes
when two types of {\it training inputs} are applied.  
When the state is ``YES'', 
the state is changed to  ``OR'' by entering the value $(1,1)$ several times in row.
On the other hand, when the state is ``OR'', 
the state is changed to  ``YES'' by entering the value $(0,1)$ several times in row.

This model can be interpreted in terms of Pavlov's dog experiment as follows.
The state ``YES'' corresponds to responding only when the dog is being fed, while  
the state ``OR'' corresponds to responding when the dog is being fed or hears the bell. 
Then the input value $(1,1)$ is interpreted as the stimulus for acquisition, i.e.,  
feeding with bell ringing, and  
$(0,1)$ is interpreted as  the stimulus for extinction, i.e., bell
ringing without feeding. 


The model of the above classical conditioning gate is expressed as   
\begin{eqnarray}
\left\{
\begin{array}{l}
x(t+1) = 
\left\{
\begin{array}{ll}
{\rm OR} &~{\rm if}~x(t-s+1)={\rm YES}, \\
         & ~~~~(v(\tau),w(\tau)) = (1,1)~ (\tau=t,t-1,\ldots,t-s+1),\\
{\rm YES} &~{\rm if}~ x(t-s+1)={\rm OR}, \\
          & ~~~~ (v(\tau),w(\tau)) = (0,1)~ (\tau=t,t-1,\ldots,t-s+1),\\
x(t)& ~{\rm otherwise},\\
\end{array}
\right.  \\ [11mm]
y(t) = 
\left\{
\begin{array}{lll}
v(t) &~{\rm if}~{\rm YES}, \\
v(t) \vee w(t) &~{\rm if}~{\rm OR}, 
\end{array}
\right.
\end{array}
\right. 
\label{eq: 1-node model}
\end{eqnarray}
where $x(t)\in\{{\rm YES}, {\rm OR}\}$ is the state, 
$v(t)\in\{0,1\}$ and $w(t)\in\{0,1\}$ are the inputs, 
$y(t)\in\{0,1\}$ is the output,
and $s\in\{1,2,\ldots\}$ is a period, called the {\it unit training time}. 
The state equation represents classical conditioning,   
while 
the output equation represents 
the resulting input-output relation at time $t$. 

The following result presents a basic property of 
(\ref{eq: 1-node model}), 
which will be utilized for training a network of classical conditioning gates.

\begin{lemma}\label{lem: state invariant property for 1 node}
Consider the classical conditioning gate in (\ref{eq: 1-node model})
with $x(t)=\bar{x}$, where  $t\in\{0,1,\ldots\}$ and
$\bar{x} \in\{{\rm YES}, {\rm OR}\}$ are arbitrarily given.
Then the following statements hold. \\
(i) If $(v(t),w(t)) = (0,0)$, 
then $y(t)=0$ and $x(t+1)=x(t)$. \\
(ii) If $(v(t),w(t)) = (1,0)$, 
then $y(t)=1$ and $x(t+1)=x(t)$. 
\end{lemma}
\begin{proof}
Trivial from (\ref{eq: 1-node model}).
\end{proof}
This shows that there exists an input value that 
sets an arbitrary value at the output while preserving the state value.

\begin{figure}[t]
	\centering
	\includegraphics{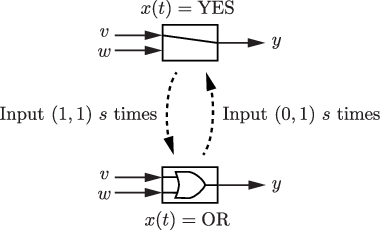}
	\vspace*{2mm}
	\caption{Classical conditioning gate.}
	\label{fig: classical conditioning device}
\end{figure}

\subsection{Network of Classical Conditioning Gates}
Now, we introduce a network of classical conditioning gates,
as shown in Fig.~\ref{fig: network}.
The network has a binary tree structure, where each gate,
except for the leftmost gates, is connected to two other gates on the input side. 
The network has $m$ layers, indexed by  
$1, 2, \ldots, m$ from the input side.
The $i$-th layer has $2^{m-i}$ gates, and thus 
the network has $\sum_{i=1}^m  2^{m-i}$ gates.
We use $n_i$ and $n$ to denote these numbers, i.e., 
$n_i= 2^{m-i}$ and $n=\sum_{i=1}^m  2^{m-i}=\sum_{i=1}^m  n_i$. 

We introduce the following notation for the network. 
The network with $m$ layers is denoted by $\Sigma(m)$.
The $j$-th gate from the top in the $i$-layer is 
called {\it node} $(i,j)$,
and let $x_{ij}(t)\in\{{\rm YES}, {\rm OR}\}$, 
$v_{ij}(t)\in\{0,1\}$, $w_{ij}(t)\in\{0,1\}$, and $y_{ij}(t)\in\{0,1\}$
 be 
the state, first input, second input, and output of node $(i,j)$, respectively. 
The pair of $v_{ij}(t)$ and $w_{ij}(t)$ is often denoted by $u_{ij}(t)\in\{0,1\}^2$.    
We use $\bar{x}_{ij}(t)$ to represent 
the flipped value of $x_{ij}(t)$: 
$\bar{x}_{ij}(t)={\rm YES}$ for ${x_{ij}(t)}={\rm OR}$, while 
$\bar{x}_{ij}(t)={\rm OR}$ for ${x_{ij}(t)}={\rm YES}$. 
   
Next, let us consider the collection of signals.
We use ${\mathbf N}:=\{(1,1), (1,2),\ldots,$ $ (m,1)\}$, which has $n$ elements,    
to represent the set of node indeces, and use  
${\mathbf N}_i \subset {\mathbf N}$ to represent the set of node indeces in layer $i$.
Let $X_i(t)\in \{{\rm YES}, {\rm OR}\}^{n_i}$,
$U_i(t)\in\prod_{j=1}^{n_i}( \{0,1\} \times \{0,1\})$, and  
$Y_i(t)\in\{0,1\}^{n_i}$ be the collective state, 
input, and output of the $i$-th layer. 
Let also $X(t):=(X_1(t),X_2(t),\ldots,X_m(t)) \in
\{{\rm YES}, {\rm OR}\}^{n_1} \times \{{\rm YES}, {\rm OR}\}^{n_2}
\times \cdots
\times \{{\rm YES}, {\rm OR}\}^{n_m}
= \{{\rm YES}, {\rm OR}\}^{n}$.
According to this notation, 
the state, input, and output of the network $\Sigma(m)$
are denoted by $X(t)$, $U_1(t)$, and $Y_m(t)$, respectively.
Note that $Y_m(t)=y_{m1}(t)$. 

Fig.~\ref{fig: example of Sigma(3)} shows an example for $m=3$. 
In this case, we have ${\mathbf N}=\{(1,1),(1,2), (1,3),$\\ $ (1,4), (2,1), (2,2), (3,1)\}$,  
$X_1(t)= (x_{11}(t), x_{12}(t), x_{13}(t), x_{14}(t))$, 
$X_2(t)= (x_{21}(t),$ $ x_{22}(t))$, 
$X_3(t)= x_{31}(t)$,
$U_1(t)= (u_{11}(t), u_{12}(t), u_{13}(t), u_{14}(t))$, 
$U_2(t)= (u_{21}(t), u_{22}(t))$, 
$U_3(t)= u_{31}(t)$,
$Y_1(t)= (y_{11}(t), y_{12}(t), y_{13}(t), y_{14}(t))$, 
$Y_2(t)= (y_{21}(t), y_{22}(t))$, and 
$Y_3(t)= y_{31}(t)$.

\begin{figure}[t]
	\centering
	\includegraphics{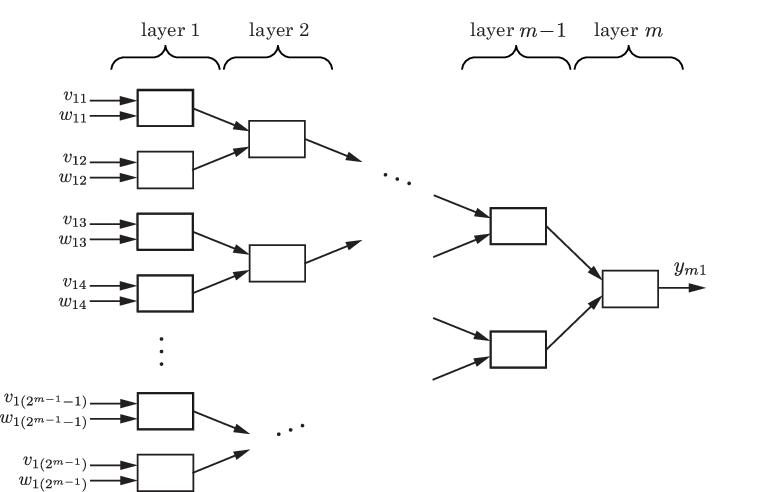}\vspace*{2mm}
	\caption{Network of classical conditioning gates.}
	\label{fig: network}
\end{figure}

\begin{figure}[t]
	\centering
	\includegraphics{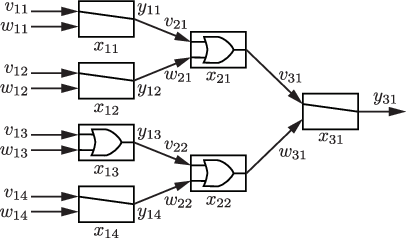}\vspace*{2mm}
	\caption{Network $\Sigma(3)$ with 
	$X(0)=({\rm YES},{\rm YES},{\rm OR},{\rm YES},{\rm OR},{\rm OR},{\rm YES})$.}
	\label{fig: example of Sigma(3)}
\end{figure}

Since a classical conditioning gate operates as either logical YES or OR,
the possible input-output relation of $\Sigma(m)$ is limited to 
logical OR of some of the inputs of $\Sigma(m)$.  
Moreover, the output $y_m(t)$ always depends on the input $v_{11}(t)$ because 
$v_{11}(t)$ propagates through nodes $(i,1)$ ($i=1,2,\ldots,m$)  
operating as logical YES or OR.  
For example, 
$y_3(0)=v_{11}(0)\vee v_{12}(0)$ for the network $\Sigma(3)$ in Fig.~\ref{fig: example of Sigma(3)}. 
This fact is formalized as follows. 
\begin{lemma}\label{lem: possible input-output relation}
Consider the network $\Sigma(m)$ with $X(t)=\bar{X}$, 
where  $t\in\{0,1,\ldots\}$ and
$\bar{X} \in\{{\rm YES}, {\rm OR}\}^n$ are arbitrarily given.
The following statements hold. \\ 
(i) The input-output relation at time $t$ is given by 
\begin{eqnarray}
y_m(t) = 
\left(\bigvee_{i\in {\mathbf J}_1} v_{1j}(t) \right)\vee
\left(\bigvee_{i\in {\mathbf J}_2} w_{1j}(t) \right)
\label{eq: y(T)}
\end{eqnarray}
for some ${\mathbf J}_1\subseteq\{1,2,\ldots,n_1\}$ and
${\mathbf J}_2\subseteq\{1,2,\ldots,n_1\}$. \\
(ii) If (\ref{eq: y(T)}) holds, then $1\in {\mathbf J}_1$. 
Moreover, if $v_{11}(t)=1$, then $y_m(t)=1$.~\hfill\hfill$\Box$
\end{lemma}
  
The following result shows a structural property that indicates 
which inputs affect node $(i,j)$.  
\begin{lemma}\label{lem: control input block}
Consider the network $\Sigma(m)$ with $X(t)=\bar{X}$ and a node 
$(i,j)\in\{2,3,\ldots,m\}\times\{1,2,\ldots,n_i\}$, 
where  $t\in\{0,1,\ldots\}$ and
$\bar{X} \in\{{\rm YES}, {\rm OR}\}^n$ are arbitrarily given.
Assume that the input $U(t)$ is divided into $2n_i$ blocks 
in the same size and let $[U(t)]_{k}\in \{0,1\}^{2^{i-1}}$ be the $k$-th block.   
Then
\begin{eqnarray}
v_{ij}(t) = f_1([U(t)]_{2j-1}),~~ 
w_{ij}(t) = f_2([U(t)]_{2j}) 
\end{eqnarray}
holds for some functions $f_1:\{0,1\}^{2^{i-1}} \to \{0,1\}$ 
and $f_2:\{0,1\}^{2^{i-1}} \to \{0,1\}$.
\end{lemma}  
\begin{proof}
It is trivial from the definition of $\Sigma(m)$. See Fig.~\ref{fig: network}.  
\end{proof}

Lemma~\ref{lem: control input block} is illustrated as follows. 
Consider the network $\Sigma(3)$ in Fig.~\ref{fig: example of Sigma(3)}
and node $(2,1)$. 
Then  $U(t)$ is divided into $4$ blocks (where $2n_2=2^2$):   
$[U(t)]_{1}=(v_{11}(t),w_{11}(t))$,
$[U(t)]_{2}=(v_{12}(t),w_{12}(t))$,
$[U(t)]_{3}=(v_{13}(t),w_{13}(t))$, and 
$[U(t)]_{4}=(v_{14}(t),w_{14}(t))$.
From Lemma~\ref{lem: control input block}, we have  
$v_{21}(t) = f_1(U_{[1]}(t))$ and 
$w_{21}(t) = f_2(U_{[2]}(t))$ for 
some $f_1:\{0,1\}^{2^{i-1}} \to \{0,1\}$ and $f_2:\{0,1\}^{2^{i-1}} \to \{0,1\}$. 
This is consistent with the interdependence between signals 
in Fig.~\ref{fig: example of Sigma(3)}.



\section{Problem Formulation}
For the network $\Sigma(m)$, we address 
the following learning problem.

\begin{problem}\label{prob: training problem}
Consider the network $\Sigma(m)$ with $X(0)=\bar{X}$, 
where $\bar{X} \in\{{\rm YES}, {\rm OR}\}^n$ is arbitrarily given.
Suppose that ${\mathbf J}_1\subseteq\{1,2,\ldots,2^m\}$ and
${\mathbf J}_2\subseteq\{1,2,\ldots,2^m\}$ are given.
Find a time $T\in\{1,2,\ldots\}$ and an input sequence 
$(U(0),U(1), \ldots, U(T-1))$
such that (\ref{eq: y(T)}) holds for $t=T$.~\hfill\hfill$\Box$
\end{problem}

Two remarks are given. 

First, the problem is not always feasible 
because (\ref{eq: y(T)}) cannot be always realized in $\Sigma(m)$.   
For example, as is seen from the output equation in (\ref{eq: 1-node model}),
$y_m(T) = v_{12}(T)\vee v_{22}(T)$ is not possible for any $T\in\{1,2,\ldots\}$.

Second, if the problem is feasible, 
there exists a vector $X^*\in\{{\rm YES}, {\rm OR}\}^n$
such that $X(T)=X^*$ implies (\ref{eq: y(T)}).
For example, consider the system $\Sigma(3)$ in Fig.~\ref{fig: example of Sigma(3)}
and the case where ${\mathbf J}_1=\{1,3,4\}$ and ${\mathbf J}_2=\{3\}$
for (\ref{eq: y(T)}).  
Then, we have $X^*=({\rm YES},{\rm YES},{\rm OR},{\rm YES},{\rm YES},{\rm OR},{\rm YES})$,
for which $X(T)=X^*$ implies 
$y_m(T) = v_{11}(T)\vee v_{13}(T) \vee w_{13}(T) \vee v_{14}(T)$.
Thus, the problem is reduced into finding the input sequence 
to steer the state to $X^*$.

\section{Learning}
Now, we present a solution to 
Problem~\ref{prob: training problem}. 

\subsection{Flipping Principle}
Let us provide a key principle, called the {\it flipping principle}, 
for solving Problem~\ref{prob: training problem}.

The following is a preliminary result to derive the flipping principle.
\begin{lemma}\label{lem: output control lemma}
Consider the network $\Sigma(m)$ with $X(t)=\bar{X}$,
where  $t\in\{0,1,\ldots\}$ and
$\bar{X} \in\{{\rm YES}, {\rm OR}\}^n$ are arbitrarily given.
Then the following statements hold. \\
(i) If $U(t)=(0,0,\ldots,0)\in\{0,1\}^{2n_1}$, 
then $y_m(t)=0$ and $X(t+1)=X(t)$. \\
(ii)  If $U(t)=(1,0,\ldots,0)\in\{0,1\}^{2n_1}$, 
then $y_m(t)=1$ and $X(t+1)=X(t)$.
\end{lemma}
\begin{proof}
In (i) and (ii), the relation
$X(t+1)=X(t)$ is proven by the network structure of $\Sigma(m)$ and
Lemma~\ref{lem: state invariant property for 1 node}, which 
states that, in (\ref{eq: 1-node model}), $x(t+1)=x(t)$ holds under   
$(v(t),w(t))=(0,0)$ or $(v(t),w(t))=(1,0)$. 
Next, Lemma~\ref{lem: possible input-output relation} (i) (in particular, (\ref{eq: y(T)}))
implies that $y_m(t)=0$ for $U(t)=(0,0,\ldots,0)$,
which proves (i). 
On the other hand, it follows from Lemma~\ref{lem: possible input-output relation} (ii)
that $y_m(t)=1$ for $U(t)=(1,0,\ldots,0)$. This proves (ii). 
\end{proof}
Lemma~\ref{lem: output control lemma} implies that 
there exists an input value for $\Sigma(m)$   
that sets an arbitrary value at the output of $\Sigma(m)$ while preserving the state value.
Note from this lemma that  
the state of $\Sigma(m)$ does not change by an input sequence 
that takes $(0,0,\ldots,0)$ and $(1,0,\ldots,0)$ at each time. 

From Lemma~\ref{lem: output control lemma}, we obtain the {\it flipping principle}  
for learning of $\Sigma(m)$.  
\begin{theorem}\label{the: flipping principle}
Consider the network $\Sigma(m)$ with $X(t)=\bar{X}$ and node $(p,q)$, 
where $t\in\{0,1,\ldots\}$ and $\bar{X}\in\{{\rm YES}, {\rm OR}\}^n$ are 
arbitrarily given.
Then the following statements hold. 
 \\
(i) There exists an input sequence 
$(U_t,U_{t+1},\ldots,U_{t+s-1})\in \{0,1\}^{2n_1}\times \{0,1\}^{2n_1} \times \cdots \times \{0,1\}^{2n_1}$ (Cartesian product of $s$ sets)
such that 
\begin{eqnarray}
x_{ij}(t+s) =\left\{
\begin{array}{lll}
\bar{x}_{ij}(t) &~{\rm if}~(i,j)= (p,q), \\
x_{ij}(t) &~{\rm if}~(i,j)\in {\mathbf N}_p \setminus \{(p,q)\}
\end{array}
\right.
\label{eq: flipping}
\end{eqnarray}
under $U(t) = U_t$, $U(t+1) = U_{t+1}$, 
$\ldots$, $U(t+s-1) = U_{t+s-1}$, where 
$s\in\{0,1,\ldots\}$ is the unit training time defined for 
 (\ref{eq: 1-node model}). \\
(ii)
An input sequence satisfying (\ref{eq: flipping}) is given by 
$(\hat{U}_{(p,q)},\hat{U}_{(p,q)}, \ldots,\hat{U}_{(p,q)})$ (constant on the time interval $\{t,t+1,\ldots,t+s-1\}$), 
where $\hat{U}_{(p,q)} \in \{0,1\}^{2n_1}$ is an input value 
which is divided into $2^{m+1-p}$ blocks in the same size and 
whose blocks are given as follows: 
\begin{eqnarray}
\begin{array}{rl}
\mbox{$(2q-1)$-th block}: & 
\left\{\begin{array}{ll} 
(0,0,0,\ldots,0) & {\rm if}~x_{pq}(t)={\rm OR}, \\ 
 (1,0,0,\ldots,0)  & {\rm if}~x_{pq}(t)={\rm YES}, \\ 
\end{array}\right. \\
\mbox{$2q$-th block}: & (1,0,0,\ldots,0), \\
\mbox{Other blocks}: & (0,0,0,\ldots,0). 
\end{array}
\end{eqnarray}
\end{theorem}
\begin{proof}
Statements (i) and (ii) are proven by  
showing that (\ref{eq: flipping}) holds for 
the input sequence $(\hat{U}_{(p,q)},\hat{U}_{(p,q)}, \ldots,\hat{U}_{(p,q)})$ 
specified in (ii).

By definition, the network $\Sigma(m)$ can be represented 
as the cascade connection of $m$ layers as shown in Fig. \ref{fig: layer expression network}.  
As we can see by comparing 	Figs.~\ref{fig: network} and \ref{fig: layer expression network},  
the entire left side of the $i$-th layer is 
the parallel system of $2n_{i}$ subsystems, denoted by $S_1,S_2,\ldots, S_{2n_{i}}$,
as shown in Fig.~\ref{fig: layer expression network 2}. 
Each subsystem is equivalent to the network of $i-1$ layers, i.e., $\Sigma(i-1)$. 
This allows us to apply Lemma~\ref{lem: output control lemma} to 
each subsystem because 
Lemma~\ref{lem: output control lemma} holds for  any $m\in\{1,2,\ldots\}$.

Now, consider node $(p,q)$. 
Suppose that $\Sigma(m)$ is represented as Fig.~\ref{fig: layer expression network 2} for $i=q$,
and let $Z_{k}(t)\in\{{\rm YES}, {\rm OR}\}^{\nu_p}$ be the state of the subsystem $S_k$, where  
$\nu_p=\sum_{i=1}^{p-1} 2^{i-1}$.
Note here that the following statements are equivalent: 
\begin{itemize}
  \item $Z_k(t+s)=Z_k(t)$ for every $k\in \{1,2,\ldots,2n_p\}$.
  \item $x_{ij}(t+s)=x_{ij}(t)$ for every $(i,j)\in {\mathbf N}_{p-1}$. 
\end{itemize}
We divide the proof  into two cases: $x_{pq}(t)={\rm OR}$ and 
$x_{pq}(t)={\rm YES}$.  
First, we address the case $x_{pq}(t)={\rm OR}$. 
If $x_{pq}(t)={\rm OR}$ and $U(t)=\hat{U}_{(p,q)}$, 
it follows from Lemmas~\ref{lem: control input block} and 
\ref{lem: output control lemma} (Lemma~\ref{lem: output control lemma} is applied to $\Sigma(p-1)$)
that $(v_{pq}(t),w_{pq}(t)) =(0,1)$, 
$(v_{pj}(t),w_{pj}(t)) =(0,0)$ for $j\in\{1,2,\ldots,n_p\}\setminus\{q\}$, 
and $Z_k(t+1)=Z_k(t)$ for $k\in \{1,2,\ldots,2n_p\}$. 
Thus if $U(t) = \hat{U}_{(p,q)}$, $U(t+1) = \hat{U}_{(p,q)}$, 
$\ldots$, $U(t+s-1) = \hat{U}_{(p,q)}$, then 
\begin{itemize}
  \item $x_{pq}(t+s)={\rm YES}=\bar{x}_{pq}(t)$,
  \item $x_{pj}(t+s)=x_{pj}(t)$ for $j\in\{1,2,\ldots,n_p\}\setminus\{q\}$,
  \item $Z_k(t+s)=Z_k(t)$ for $k\in \{1,2,\ldots,2n_p\}$.
\end{itemize}
This implies (\ref{eq: flipping}). 
The other case $x_{pq}(t)={\rm YES}$ can be proven in 
the same manner. 
The only difference is that $(v_{pq}(t),w_{pq}(t)) =(1,1)$ holds
when $x_{pq}(t)={\rm OR}$ and $U(t)=\hat{U}_{(p,q)}$.
\end{proof}

\begin{figure}[t]
	\centering
	\includegraphics{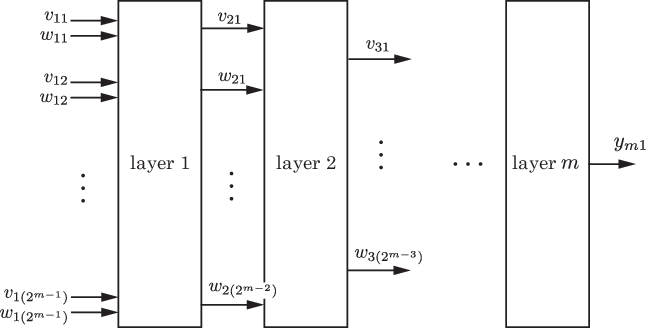}\vspace*{4mm}
	\caption{Layer-based representation of network $\Sigma(m)$.}
	\label{fig: layer expression network}
\end{figure}

\begin{figure}[t]
	\centering
	\includegraphics{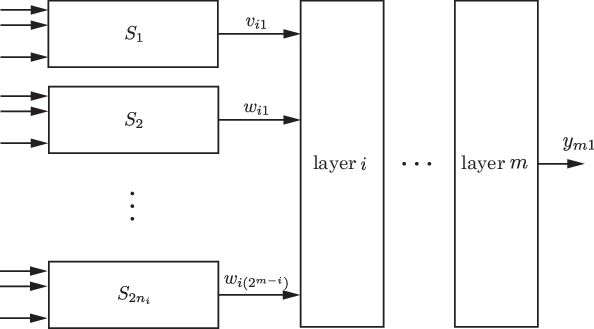}\vspace*{4mm}
	\caption{Another layer-based representation of network $\Sigma(m)$.}
	\label{fig: layer expression network 2}
\end{figure}


Theorem~\ref{the: flipping principle} implies that
we can flip the state of any node 
while preserving the state of the other nodes in 
the layer to which the node to be flipped belongs and its upstream layers.

\begin{example} \label{example 1}
Consider the network $\Sigma(3)$ in Fig.~\ref{fig: example of Sigma(3)},
where $X(0)=({\rm YES},{\rm YES},{\rm OR},$ ${\rm YES},{\rm OR},{\rm OR},{\rm YES})$. 
By the input sequence $(\hat{U}_{(2,1)},\hat{U}_{(2,1)}, \ldots,\hat{U}_{(2,1)})$
for $\hat{U}_{(2,1)}=((1,0,0,0),(0,0,0,0))$, 
the state of node $(2,1)$ is flipped while preserving the states of 
nodes $(1,1)$, $(1,2)$, $(1,3)$, $(1,4)$, and $(2,2)$. 
Fig.~\ref{fig: example of intermediated state} shows 
$\Sigma(3)$ with the resulting state $X(s)$.~\hfill\hfill$\Box$  
\end{example}

\subsection{Learning Algorithm}
Theorem~\ref{the: flipping principle} implies that 
we can steer the state of the network $\Sigma(m)$ from any value to any value
by flipping the state of each node one by one from the upstream node.  
Based on this idea, we obtain the following algorithm to  solve 
Problem \ref{prob: training problem}.

\vspace*{2mm}
\noindent
\underline{\bf Algorithm 1} \vspace*{-1mm}
\begin{description}
\item [(Step 1)]
Let $X^*\in \{{\rm YES}, {\rm OR}\}^{n}$ be a state associated with 
the desired input-output relation in (\ref{eq: y(T)})
and let $x^*_{ij}\in \{{\rm YES}, {\rm OR}\}$ be its element corresponding to node $(i,j)$.  
Let also $k:=0$ and $\hat{\mathbf N}:={\mathbf N}$.

\item [(Step 2)]
Pick the minimum pair $(i,j)$ from $\hat{\mathbf N}$ in lexicographical order.  

\item [(Step 3)] 
If $x_{ij}(ks)\neq x^*_{ij}$,  
apply the input sequence $(\hat{U}_{(i,j)},\hat{U}_{(i,j)},\ldots,\hat{U}_{(i,j)})$
to the network $\Sigma(m)$, 
i.e., $U(ks)=\hat{U}_{(i,j)}, U(ks+1)=\hat{U}_{(i,j)}, \ldots, U(ks+s-1)=\hat{U}_{(i,j)}$, and 
let $k:=k+1$. 
 
\item [(Step 4)] Let $\hat{\mathbf N}:=\hat{\mathbf N}\setminus\{(i,j)\}$. 
If $\hat{\mathbf N}\neq \emptyset$, goto Step 2; otherwise, halt. 
\end{description}
 
In the algorithm, $k$ is a variable to count 
the number of nodes whose state is flipped, and 
$\hat{\mathbf N}$ is the list of the nodes for which 
the algorithm has never checked whether their state needs to be flipped or not.
In Step 1, $X^*$ is defined from (\ref{eq: y(T)}) 
and $k$ and $\hat{\mathbf N}$ are initialized.
Step 2 picks a node $(i,j)$ from the list $\hat{\mathbf N}$.   
Step 3 checks whether the state of node $(i,j)$ has to be flipped or not; 
if it has to be flipped, the state is flipped by applying the training input sequence specified in 
Theorem~\ref{the: flipping principle}. 
Note here that $s$ steps of actual time elapsed 
while applying the input sequence to $\Sigma(m)$. 
In Step 4, node $(i,j)$ is removed from the list $\hat{\mathbf N}$. 
In addition, 
the terminate condition is checked; 
if $\hat{\mathbf N}$ is empty, the algorithm terminates;
otherwise, the above procedure is performed for a remaining node in the list $\hat{\mathbf N}$.

For the algorithm, we obtain the following result. 
\begin{theorem}
Consider Problem~\ref{prob: training problem}. 
Assume that there exists a  vector $X^*\in\{{\rm YES}, {\rm OR}\}^n$
such that (\ref{eq: y(T)}) is equivalent to $X(T)=X^*$. 
Let $k^*\in\{0,1,\ldots\}$ be the value of $k$ when Algorithm 1 terminates 
and ${\mathbb U}(k)$ ($k=0,1,\ldots k^*-1$) are the input sequence generated in Step 3
of Algorithm 1. 
Then a solution to the problem is given by 
$T=k^* s$ and $({\mathbb U}(0),{\mathbb U}(1),\ldots,{\mathbb U}(k^*-1))$.~\hfill\hfill$\Box$  
\end{theorem}

\section{Example}
Consider the network $\Sigma(3)$ in Fig.~\ref{fig: example of Sigma(3)},  
where $X(0)=({\rm YES},{\rm YES},{\rm OR},{\rm YES},{\rm OR},{\rm OR},$ ${\rm OR})$. 
Assume that $s=3$.
Let us show how Algorithm 1 solves Problem~\ref{prob: training problem}
for ${\mathbf J}_1=\{1,3,4\}$ and ${\mathbf J}_2=\{3\}$.

In Step 1, we have $X^*=({\rm YES},{\rm YES},{\rm OR},{\rm YES},{\rm YES},{\rm OR},{\rm YES})$. 
Then the algorithm generates the input sequence 
${\mathbb U}(0)=(\hat{U}_{(2,1)},\hat{U}_{(2,1)},\hat{U}_{(2,1)})$
to flip the state of node $(2,1)$ in Step 3 at $k=0$. 
The result is shown in Fig.~\ref{fig: example of intermediated state}. 
Next, the algorithm generates the input sequence 
${\mathbb U}(1)=(\hat{U}_{(3,1)},\hat{U}_{(3,1)},\hat{U}_{(3,1)})$ in Step 3
at $k=1$, which flips the state of node $(3,1)$. 
As the result, we have the system in Fig.~\ref{fig: example of final state} with  
the output  $y_m(2s) = v_{11}(2s)\vee v_{13}(2s) \vee w_{13}(2s) \vee v_{14}(2s)$.

\begin{figure}[t]
	\centering
	\includegraphics{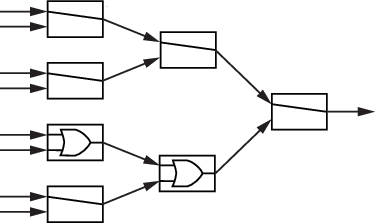}\vspace*{2mm}
	\caption{Network $\Sigma(3)$ at the state $X(s)$, resulted  by
	 the input sequence ${\mathbb U}(0)=(\hat{U}_{(2,1)},\hat{U}_{(2,1)},\hat{U}_{(2,1)})$ that flips  
	 the state of node $(2,1)$.}
	\label{fig: example of intermediated state}
\end{figure}

\begin{figure}[t]
	\centering
	\includegraphics{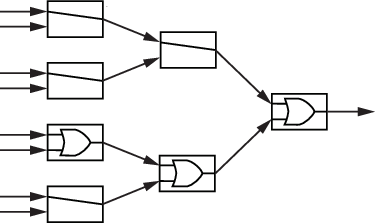}\vspace*{2mm}
	\caption{Network $\Sigma(3)$ at the state $X(2s)$, resulted  by
	 the input sequence ${\mathbb U}(1)=(\hat{U}_{(3,1)},\hat{U}_{(3,1)},\hat{U}_{(3,1)})$ that flips  
	 the state of node $(3,1)$.}
	\label{fig: example of final state}
\end{figure}

\section{Conclusion}
We have presented a learning method for a network of 
nodes each of which can implement classical conditioning.
Based on the principle that the state of any node can 
be flipped while preserving the state of some other nodes,
an algorithm has been derived. 

The proposed algorithm can be applied only to the networks with a tree structure. In the future, we plan to generalize our framework to handle a more general class of networks. 
It is also interesting to address other types of gates, for example, whose state takes logical operators that are not necessarily YES and OR.

\section*{Acknowledgment}
This work was supported by
Grant-in-Aid for Transformative Research Areas (A) 20H05969 
from the Ministry of Education, Culture, Sports, Science and Technology of Japan.

\section*{Declarations}
Conflict of interest: The author declares that he has no conflict of interest.

\end{document}